\def\year{2019}\relax
\documentclass[letterpaper]{article} 
\usepackage{aaai19}  
\usepackage{times}  
\usepackage{helvet}  
\usepackage{courier}  
\usepackage{url}  
\usepackage{graphicx}  
\usepackage{epsfig}
\usepackage{amsmath}
\usepackage{amssymb}
\usepackage{amsthm}
\usepackage{algorithm,algorithmic}
\usepackage{subfigure}
\usepackage{amsfonts}
\usepackage{multirow,array,bm}
\usepackage{color}
\usepackage{enumerate}
\usepackage{microtype}      
\usepackage{caption}
\usepackage{epstopdf}
\newtheorem{myDef}{Definition}

\newtheorem{myProp}{Proposition}
\newtheorem{myLem}{Lemma}
\begin{document}
%
\title{Energy Confused Adversarial Metric Learning for Zero-Shot Image Retrieval and Clustering}
\author{Binghui Chen, Weihong Deng\\
Beijing University of Posts and Telecommunications\\
{\tt\small chenbinghui@bupt.edu.cn, whdeng@bupt.edu.cn}
}
\maketitle
\begin{abstract}
Deep metric learning has been widely applied in many computer vision tasks, and recently, it is more attractive in \emph{zero-shot image retrieval and clustering}(ZSRC) where a good embedding is requested such that the unseen classes can be distinguished well. Most existing works deem this 'good' embedding just to be the discriminative one and thus race to devise powerful metric objectives or hard-sample mining strategies for leaning discriminative embedding. However, in this paper, we first emphasize that the generalization ability is a core ingredient of this 'good' embedding as well and largely affects the metric performance in zero-shot settings as a matter of fact. Then, we propose the Energy Confused Adversarial Metric Learning(ECAML) framework to explicitly optimize a robust metric. It is mainly achieved by introducing an interesting Energy Confusion regularization term, which daringly breaks away from the traditional metric learning idea of discriminative objective devising, and seeks to 'confuse' the learned model so as to encourage its generalization ability by reducing overfitting on the seen classes. We train this confusion term together with the conventional metric objective in an adversarial manner. Although it seems weird to 'confuse' the network, we show that our ECAML indeed serves as an efficient regularization technique for metric learning and is applicable to various conventional metric methods. This paper empirically and experimentally demonstrates the importance of learning embedding with good generalization, achieving state-of-the-art performances on the popular CUB, CARS, Stanford Online Products and In-Shop datasets for ZSRC tasks. \textcolor[rgb]{1, 0, 0}{Code available at http://www.bhchen.cn/}.
\end{abstract}

\section{1. Introduction}\label{sec_intro}
Since \emph{zero-shot learning} (ZSL) removes the limitation of category-consistency between training and testing sets, it turns to be more attractive where the model is required to learn concepts from \emph{seen} classes and then enables to distinguish the \emph{unseen} classes. ZSL has been widely explored in image classification \cite{changpinyo2016synthesized,fu2015transductive,zhang2015zero,zhang2017learning} and retrieval tasks \cite{dalton2013zero,shen2018zero,oh2016deep}, \emph{etc}. In this paper, we focus on \emph{zero-shot image retrieval and clustering} tasks(ZSRC).

\begin{figure}[t]
  \centering
  \vspace{-2em}
  \begin{minipage}{0.9\linewidth}
  \centering
  \includegraphics[width=1\linewidth]{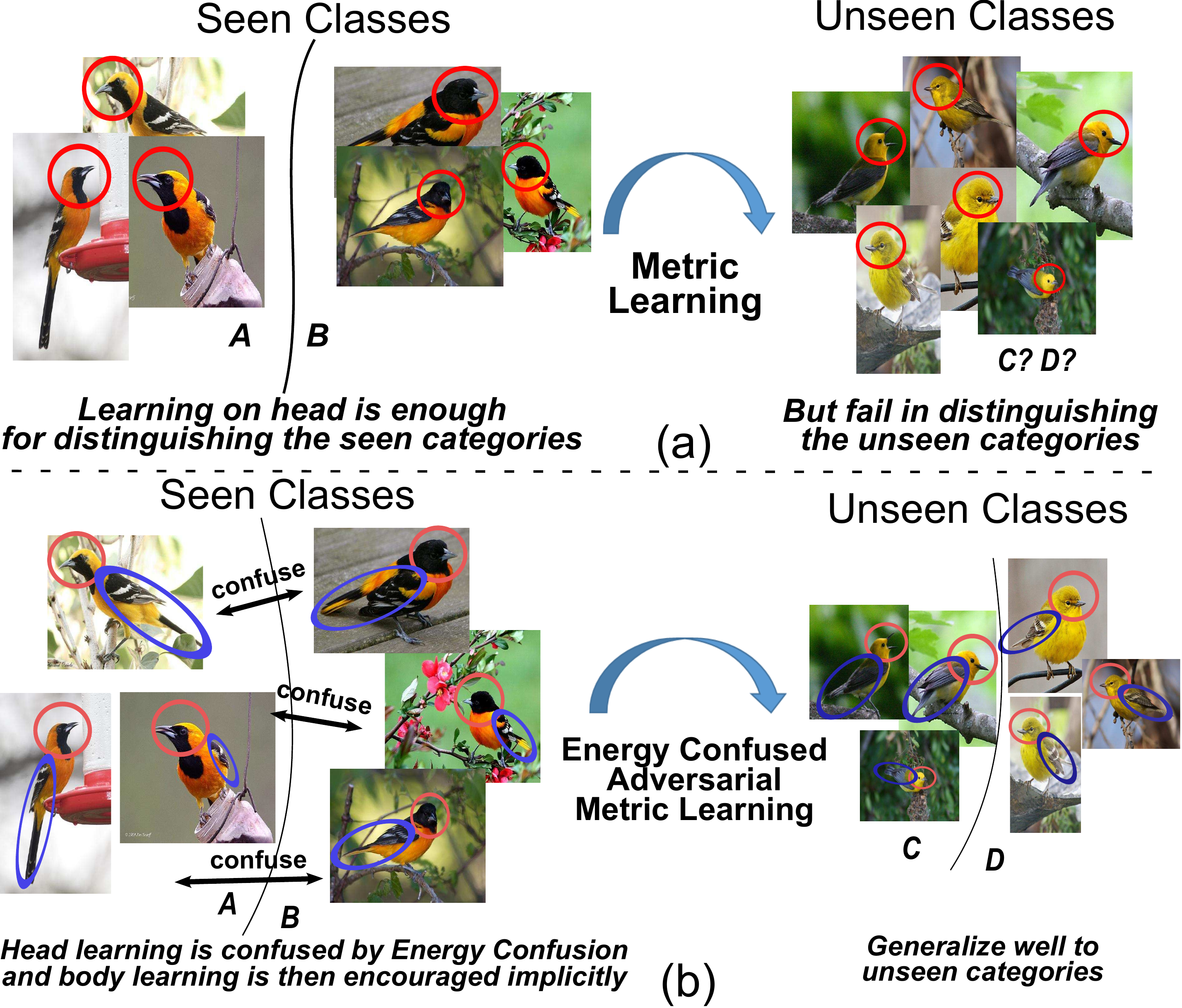}\\
  \captionsetup{font={footnotesize}}
  \vspace{-1em}
  \caption{Comparisons of conventional metric learning methods and our Energy Confused Adversarial Metric Learning (ECAML). In (a), the deep model optimized by conventional metric learning will selectively learn head knowledge which is the easiest one to reduce the current training error and omit other helpful concepts, but the testing instances cannot be distinguished well by the head. In (b), the Energy Confusion(EC) term among different classes is introduced so as to make the biased head-based metric confused about itself, then as the training going, EC will regularize this metric to explore other complementary knowledge (even if this knowledge is not discriminative enough for the \emph{seen} classes, it might be helpful for the \emph{unseen} classes) and thus improve the generalization ability.}
  \label{fig_intro}
  \end{minipage}
  \vspace{-2em}
\end{figure}
In order to accurately retrieve and cluster the \emph{unseen} classes, most existing works employ \emph{Deep Metric Learning} to optimize a good embedding, such as exploring tuple-based loss functions \cite{sun2014deep,Yuan_2017_ICCV,Wu_2017_ICCV,Schroff2015FaceNet,oh2016deep,wang2017deep,Huang2016Local,Sohn2016npair} and proposing efficient hard-sample mining strategies \cite{kumar2017smart,Wu_2017_ICCV,Schroff2015FaceNet}, \emph{etc}. However, the above methods deem this 'good' embedding just to be the discriminative one and then concentrate on the discriminative learning over the \emph{seen} classes, but neglect the importance of the generalization ability of the learned metric which is significant in ZSRC as well, as a result, without robustness constraining they are easily subject to concepts overfitting problem on the \emph{seen} classes and some helpful or general knowledge for \emph{unseen} classes may have been left out with a high probability.

To be specific, in ZSRC, we emphasize that the generalization ability of the learned embedding is seriously affected by the following problem: ``the biased learning behavior of deep models'', concretely, as illustrated in Fig.\ref{fig_intro}.(a)\footnote{In fact, the learned partial biased knowledge is more complicated and cannot be easily illustrated in figure, here for intuitive understanding, we translate it into some single body-part knowledge.}, for a functional learner parameterized by CNN, to correctly distinguish classes A and B, it will selectively learn the partial biased attributes concepts that are the easiest ones to reduce the current training loss over the \emph{seen} classes (here head knowledge is enough to sperate class A from B and thus is learned), instead of learning all-sided details and concepts, thus yielding over-fitting on \emph{seen} classes and generalizing worse to \emph{unseen} ones (classes C and D). In another word, in order to correctly recognize classes, deep networks easily learn to focus on surface statistical regularities rather than more general abstract concepts.

Therefore, when learning the embedding as in the aforementioned conventional metric learning methods, this issue objectively exists and impedes the learning of the desired good embedding. And without explicit and benign robustness constraint, the learned embedding is unable to generalize well to the \emph{unseen} classes. Most ZSRC works ignore the importance to learn robust descriptors. To this end, proposing efficient regularization method for conventional metric learning to learn metrics with good generalization is important, especially in ZSRC tasks.

In this paper, we propose the \textbf{\emph{Energy Confused Adversarial Metric Learning}} (ECAML) framework, an elegant regularization strategy, to alleviate the problem of generalization in ZSRC tasks by randomly confusing the learned metric during each iteration. It is mainly achieved by a novel and simple \emph{Energy Confusion} (EC) term which is 'plug and play' and can be generally applied to many existing deep metric learning approaches. Concretely, this confusion term plays an adversary role against the conventional metric learning objective, which intends to minimize the expected value of the Euclidean distances between the paired images from two different categories. As illustrated in Fig.\ref{fig_intro}.(b), confusing the biased head-based metric will make the model less discriminative on the \emph{seen} classes by reducing its dependence on head learning and thus give it chances of exploring other complementary and general knowledge, preventing overfitting on the \emph{seen} classes and improving the generalization ability of the embedding in an adversarial manner. In another word, the EC term allows the SGD solver to escape from the 'bad' local-minima region induced by the \emph{seen} classes and to explore more for the robust one. The main contributions of this work can be summarized as follows:

\begin{small}
\begin{itemize}
  \item We emphasize that the crucial issue to ZSRC, i.e. \emph{the biased learning behavior of deep model}, is the key stumbling block of improving the generalization ability of the learned embedding.
  \item We propose \textbf{\emph{Energy Confused Adversarial Metric Learning}}(ECAML) framework to reinforce the robustness of embedding in an adversarial manner. The Energy Confusion(EC) term is 'plug and play' and can work in conjunction with many existing metric methods. To our knowledge, it is the first work to introduce confusion for deep metric learning.
  \item Extensive experiments have been performed on several popular datasets for ZSRC, including CARS\cite{Krause20133D}, CUB\cite{Wah2011The}, Stanford Online Products\cite{oh2016deep} and In-shop\cite{liu2016deepfashion}, achieving state-of-the-art performances.
\end{itemize}
\end{small}
\vspace{-1em}
\section{2. Related Work}
\textbf{Zero-shot setting}: ZSL has been widely explored in many computer vision tasks, such as image classification\cite{changpinyo2016synthesized,fu2015transductive,zhang2015zero} and image retrieval\cite{dalton2013zero,shen2018zero}. Most of these ZSL methods are capable of exploiting the extra auxiliary supervision information of the \emph{unseen} classes (e.g. word representations of semantic name), thus aligning the learned features in an explicit manner. However in real applications, collecting and labelling these auxiliary information is time-consuming and impractical. Our ECAML concentrates on a more actual scene where there are only \emph{seen} class labels available.

\textbf{Deep metric learning for ZSRC}: The commonly used contrastive\cite{sun2014deep} and triplet loss\cite{Schroff2015FaceNet} have been broadly studied. Additionally, there are some other deep metric learning works: Smart-mining\cite{kumar2017smart} combines local triplet loss and global loss to optimize the deep metric with hard-samples mining. Sampling-Matters\cite{Wu_2017_ICCV} proposes distance weighted sampling strategy. Angular loss\cite{wang2017deep} optimizes a triangle-based angular function. Proxy-NCA\cite{Movshovitz-Attias_2017_ICCV} explains why popular classification loss works from a proxy-agent view, and its implementation is very similar to Softmax. ALMN\cite{chen2018almn} proposes to generate geometrical virtual negative point instead of employing hard-sample mining for learning discriminative embedding. However, all the above methods are to cope with the metric by designing discriminative losses or exploring sample-mining strategies, thus suffer from the aforementioned issue easily. Additionally, HDC\cite{Yuan_2017_ICCV} employs the cascaded models and selects hard-samples from different levels and models. BIER loss\cite{Opitz_2017_ICCV,opitz2018deep} adopts the online gradients boosting methods. These methods try to improve the performances by resorting to the ensemble idea. Different from all these methods, our ECAML has a clear object of improving the generalization ability of the learned metric by introducing the Energy Confusion regularization term.

\textbf{Regularization technique}: Regularization methods sometimes are important for deep models as the deep models are more likely to be data-driven. There are some works injecting random noise into deep nets so as to ensure the robust training, such as Bengio \emph{et al}.\cite{bengio2013estimating} and Gulcehre \emph{et al}.\cite{gulcehre2016noisy} add noise in the ReLU and Sigmoid activation functions respectively, Blundell \emph{et al}.\cite{blundell2015weight}, Graves\cite{graves2011practical} and Neelakantan \emph{et al}.\cite{neelakantan2015adding} add noise in weights and gradients respectively. Moreover, some research works intend to regularize the deep models at the top layer, i.e. Softmax classifier layer, for example, Szegedy \emph{et al}.\cite{szegedy2016rethinking} propose label-smoothing regularization technique for training deep models, Xie \emph{et al}.\cite{xie2016disturblabel} propose label-disturbing technique for improving the generalization ability of the deep models and, Chen \emph{et al}.\cite{chen2017noisy} inject annealed noise into the softmax activations so as to boost the generalization ability by postponing the early Softmax saturation behavior. However, different from these above methods which are mainly devised for classification tasks and applicable to the Softmax classifier layer, our ECAML aims to promote the generalization ability of the metric learning in ZSRC tasks, and it is achieved by training the EC term in an adversarial manner.
\vspace{-0.5em}
\section{3 Notations and Preliminaries}\label{sec_natation}
In this section, we review some notations and the necessary preliminaries on the relation between semimetric and RKHS kernels for later convenience, which will be used to interpret the differences between our EC (Sec.4.1\ref{sec_EC}) and some other existing methods, i.e. \emph{general energy distance} and \emph{maximum mean discrepancy}.

If not specified, we will assume that $\mathcal{Z}$ is any topological space where the Borel measures can be defined. Denote by $\mathcal{M}(\mathcal{Z})$ the set of all finite signed Borel measures on $\mathcal{Z}$, and by $\mathcal{M}_{+}^{1}(\mathcal{Z})$ the set of all Borel probability measures on $\mathcal{\mathcal{Z}}$.
\begin{myDef}
(RKHS) Let $\mathcal{H}$ be a Hilbert space of real-valued functions defined on $\mathcal{Z}$. A function $k$ : $\mathcal{Z}\times{\mathcal{Z}}\rightarrow{R}$ is called a reproducing kernel of $\mathcal{H}$, if (i) $\forall{z} \in \mathcal{Z}$, $k(\cdot,z) \in \mathcal{H}$, and (ii) $\forall{z} \in \mathcal{Z}$, $\forall{f} \in \mathcal{H}$, $<f,k(\cdot,z)>_{\mathcal{H}}~=~f(z)$. If $\mathcal{H}$ has a reproducing kernel, it is called a \emph{reproducing kernel Hilbert space} (RKHS).
\end{myDef}
\begin{myDef}
(Semimetric) Let $\mathcal{Z}$ be a nonempty set and let $\rho:\mathcal{Z}\times{\mathcal{Z}}\rightarrow{[0,+\infty)}$ be a function such that $\forall{z},z^{'}\in{\mathcal{Z}}$, (i) $\rho(z,z^{'})=0$ iff $z=z^{'}$ and (ii) $\rho(z,z^{'})=\rho(z^{'},z)$. $(\mathcal{Z},\rho)$ is called a semimetric space and $\rho$ is a semimetric.
\end{myDef}
\begin{myDef}\label{def3}
(Negative type) Semimetric space $(\mathcal{Z},\rho)$ is said to have negative type if $\forall{n}\geq2,z_{1},\ldots,z_{n}\in{\mathcal{Z}}$, and $\alpha_{1},\ldots,\alpha_{n}\in{\mathbb{R}}$, with $\sum_{i=1}^{n}\alpha_{i}=0$, $\sum_{i=1}^{n}\sum_{j=1}^{n}\alpha_{i}\alpha_{j}\rho(z_{i},z_{j})\leq0$.
\end{myDef}

Then we have the following propositions, which are derived from \cite{van2012harmonic}.
\begin{myProp}
If $\rho$ satisfies Def.\ref{def3}, then so does $\rho^{q}$, where $0<q<1$.
\end{myProp}
\begin{myProp}
$\rho$ is a semimetric of negative type iff there exists a $\mathcal{H}$ and an injective map $\varphi:\mathcal{Z}\rightarrow{\mathcal{H}}$, such that\label{prop2}
\begin{equation}\label{eq_prop2}
  \rho(z,z^{'})=\|\varphi(z)-\varphi(z^{'})\|^{2}_{\mathcal{H}}
\end{equation}
\end{myProp}
This shows that $(\mathbb{R}^{d},\|\cdot-\cdot\|^{2})$ is of negative type, and by taking $q = 1/2$, we conclude that all Euclidean spaces are of negative type\cite{sejdinovic2012hypothesis,sejdinovic2013equivalence}
, which will be used to reason our Energy Confusion term. Then we also show that the semimetrics of negative type and symmetric positive definite kernels are in fact closely related by the following Lemma(for more details please refer to \cite{van2012harmonic}).
\begin{myLem}For a nonempty $\mathcal{Z}$, let $\rho$ be a semimetric on $\mathcal{Z}$. Let $z_{0}\in{\mathcal{Z}}$, and denote $k(z,z^{'})=\frac{1}{2}(\rho(z,z_{0})+\rho(z^{'},z_{0})-\rho{(z,z^{'}))}$. Then k is positive definite iff $\rho$ is of negative type.\label{lem1}
\end{myLem}
We call the kernel defined above \emph{distance-induced kernel} and, it is induced by the semimetric $\rho$ and centered at $z_{0}$. By varying the point at the center $z_{0}$, we obtain a kernel family $\mathcal{K}_{\rho}={\frac{1}{2}[\rho(z,z_{0})+\rho(z^{'},z_{0})-\rho(z,z^{'})]}_{z_{0}\in{\mathcal{Z}}}$, induced by $\rho$. Then we can always express Eq.\ref{eq_prop2} in terms of the canonical feature map for RKHS $\mathcal{H}_{k}$ as the following proposition.
\begin{myProp}
Let $(\mathcal{Z},\rho)$ be a semimetric space of negative type, and $k\in{\mathcal{K}}_{\rho}$. Then:\label{prop3}
\begin{enumerate}
  \item $k$ is nondegenerate, i.e. the Aronszajn map $z\rightarrow{k(\cdot,z)}$ is injective.
  \item $\rho(z,z^{'})=k(z,z)+k(z^{'},z^{'})-2k(z,z^{'})=\|k(\cdot,z)-k(\cdot,z^{'})\|^{2}_{\mathcal{H}_{k}}$
\end{enumerate}
\end{myProp}
For the above valid $\rho$, we say that $k$ generates $\rho$. And the above proposition implies that the Aronszajn map $z\rightarrow{k(\cdot,z)}$ is an isometric embedding of a metric space $(\mathcal{Z},\rho^{1/2})$ into $\mathcal{H}_{k}$, for each $k\in{\mathcal{K}_{\rho}}$. Lem.\ref{lem1} and Prop.\ref{prop3} reveal the general link between semimetrics of negative type and RKHS kernels in different views. By taking some special cases of $\rho$ and $k$, we are able to elucidate our EC in the following sections.
\section{4. Proposed Approach}
\subsection{4.1 Energy Confusion}\label{sec_EC}
As discussed in (Sec.1\ref{sec_intro}), without taking the generalization ability into consideration explicitly, simply optimizing a discriminative objective metric functions or applying hard-sample mining strategies like in most existing metric learning works wouldn't lead a robust metric for ZSRC tasks, since the 'biased learning behavior of deep models' will mostly force the network to fit the surface statistical regularities rather than the more general abstract concepts, i.e. it will only highlight the concepts that are discriminative for the \emph{seen} classes instead of keeping all-sided information, resulting in overfitting on the \emph{seen} categories and limiting the generalization ability of the learned embedding.

Consider that the biased learning behavior is actually induced by the nature of model training since in order to correctly distinguish different \emph{seen} classes, the deep metric has to be confident about the feature distribution prediction over the current \emph{seen} classes as far as possible(e.g. features of different classes should be far away from each other) and as a result, only the partial biased knowledge that are discriminative to separate \emph{seen} categories as shown in Fig.\ref{fig_intro} are captured while other potentially helpful knowledge are omitted. To this end, \textbf{a natural solution is to introduce an opposite optimizing objective, i.e. a feature distribution confusion term, into the conventional metric learning phase so as to 'confuse' the network and reduce the over-confident predictions of distances between feature distributions on the \emph{seen} classes}. Specifically, denote the input features by $\{x_{i}\}_{i=1}^{N}$, the corresponding label inputs by $\{y_{i}\}^{N}_{i=1}, y_{i}\in{[1\ldots{C}]}$, where $C$ is the number of \emph{seen} classes. The conventional metric optimizing goal is to make the distance measurement $D(x_{i},x_{j})$ as large as possible if $y_{i}\neq{y_{j}}$, otherwise as small as possible, and it can be formulated as:
\vspace{-0.2em}
\begin{equation}\label{eq_metric_learning}
  \theta_{f}=\arg\min_{\theta_{f}}L_{m}(\theta_{f};T,D)
\end{equation}
\vspace{-0.7em}

where $L_{m}$ is some specific metric loss function, $T$ indicates some instance-tuple, e.g. contrastive tuple $T(x_{i},x_{j})$ \cite{sun2014deep}, triplet tuple $T(x_{i},x_{i^{+}},x_{i^{-}})$ \cite{Schroff2015FaceNet} or N-Pair tuple $T(x_{i},x_{i^{+}},x_{i^{-}_{1}},\ldots,x_{i^{-}_{N-2}})$ \cite{Sohn2016npair}, $D$ is the distance distribution measurement, e.g. Euclidean measurement\cite{oh2016deep,Yuan_2017_ICCV,Huang2016Local,Schroff2015FaceNet,Wu_2017_ICCV} or inner-product measurement\cite{Opitz_2017_ICCV,Sohn2016npair}, and $\theta_{f}$ is the metric parameters to be learned.

Therefore, in order to prevent the biased learning behavior by confusing the feature distribution learning, \textbf{\emph{we would like to learn $\theta_{f}$ that make the feature distributions from different classes closer}} when under some specific $\{L,T,D\}$. It seems that the commonly adopted family of $f$-$divergence$ for measuring the difference between two probability distributions might be a suitable choice, such as KL-divergence\cite{kullback1951information}, Hellinger-distance\cite{Hellinger1909} and Total-variation-distance, however, we emphasize that they cannot be directly applied here since they mostly work with the probability measure (where $\sum_{k}x_{i,k}=1$) but our confusion goal is based on the statistical distance between two random vectors following some probability distributions. To this end, we propose the \textbf{\emph{Energy Confusion}} term as follows:
\vspace{-0.5em}
\begin{align}\label{eq_energy_confusion}
L_{ec}(\theta_{f};X_{I},X_{J})=&\mathbb{E}_{\widetilde{X_{I}}\widetilde{X_{J}}}(\|\widetilde{X_{I}}-\widetilde{X_{J}}\|^{2}_{2})\nonumber\\
=&\sum_{i,j}p_{i,j}\|x_{i}-x_{j}\|^{2}_{2}
\end{align}
\vspace{-0.5em}

where $\mathbb{E}$ indicates the expected value, $X_{I},X_{J}$ are two different class sets, $\widetilde{X_{I}},\widetilde{X_{J}}$ are random feature vectors which obey some certain distribution, $x_{i},x_{j}$ are the corresponding feature observations and $p_{i,j}$ is the joint probability. Since during training the samples are uniformly sampled and the classes are independent, we have $\widetilde{X_{I}}\sim{Uniform(X_{I})},\widetilde{X_{J}}\sim{Uniform(X_{J})}$ and $p_{i,j}=p_{i}p_{j}=\frac{1}{N_{I}}\frac{1}{N_{J}}$. In this case, $\{L,T,D\}$ are expected value function, contrastive tuple and Euclidean measurement respectively.

From Eq.\ref{eq_energy_confusion}, one can observe that the EC term intends to minimize the distance expected value between different classes so as to confuse the metric. As discussed above, the learned embedding represents the learned concepts to some extend, and the more accurate the prediction of distance on the \emph{seen} classes, the greater the risk of concepts overfitting. EC serves as a regularization term that would like to prevent the model being over-confident about the \emph{seen} classes and mitigate the biased learning issue by avoiding the learner being stuck in the training-data-specific concepts. In another word, the metric learning is regularized by explicitly reducing model's dependence on the partial biased knowledge, and this is mainly achieved by the idea of feature distribution confusion. Moreover, 'confusing' also gives SGD solver chances of escaping from the 'partial' and 'bad' local-minima induced by the \emph{seen} instances, and then exploring other solution regions for the more 'general' ones.

\textbf{Discussion}: Inferred from the above analysis, it seems that the commonly used \emph{general energy distance(GED)} and \emph{maximum mean discrepancy}(MMD) might be also useful here for confusing the network by pushing different feature distributions closer. However, we will bridge our EC with these two methods, and illuminate the significance of our EC by theoretically accounting for why these two methods cannot be directly applied here.

\textbf{Relation to GED}: Let $(\mathcal{Z},\rho)$ be a semimetric space of negative type, and let $P,Q\in{\mathcal{M}^{1}_{+}(\mathcal{Z})\bigcap{\mathcal{M}^{1}_{\rho}(\mathcal{Z})}}$, then the \emph{general energy distance}(GED) between $P$ and $Q$, w.r.t $\rho$ is:
\begin{equation}\label{eq_ged}
  \small{D_{E,\rho}(P,Q)=2\mathbb{E}_{\widetilde{P}\widetilde{Q}}\rho(\widetilde{P},\widetilde{Q})-\mathbb{E}_{\widetilde{P}\widetilde{P}^{'}}\rho(\widetilde{P},\widetilde{P}^{'})-\mathbb{E}_{\widetilde{Q}\widetilde{Q}^{'}}\rho(\widetilde{Q},\widetilde{Q}^{'})}
\end{equation}
where $\widetilde{P},\widetilde{P}^{'}\stackrel{i.i.d}{\thicksim}P$ and $\widetilde{Q},\widetilde{Q}^{'}\stackrel{i.i.d}{\thicksim}Q$. $D_{E,\rho}$ is a general extension of \emph{energy distance}\cite{szekely2004testing,szekely2005new} on metric space. Then we have:
\begin{myLem}\label{lem2}
For two different class sets $X_{I},X_{J}\in\mathcal{M}^{1}_{+}(\mathcal{Z})\bigcap{\mathcal{M}^{1}_{\rho}(\mathcal{Z})}$, let $\rho$ be squared Euclidean metric, i.e. $\|\cdot-\cdot\|^{2}_{2}$, then:
\vspace{-1em}
\begin{equation}
  L_{ec}(\theta_{f};X_{I},X_{J})\geq{\frac{1}{2}D_{E,\rho}(X_{I},X_{J})}\nonumber
\end{equation}
\end{myLem}
\begin{proof}
from Prop.\ref{prop2}, if $\rho$ is the squared Euclidean metric, we have $(\mathcal{Z},\rho)$ is of negative type, thus from Eq.\ref{eq_ged}
\begin{small}
\begin{align}\label{eq_ed}
\frac{1}{2}D_{E,\rho}(X_{I},X_{J})=\mathbb{E}(\|\widetilde{X}_{I}-\widetilde{X}_{J}\|^{2}_{2})-\frac{1}{2}\{\mathbb{E}(\|\widetilde{X}_{I}-\widetilde{X}_{I}^{'}\|^{2}_{2})\nonumber\\
+\mathbb{E}(\|\widetilde{X}_{J}-\widetilde{X}_{J}^{'}\|^{2}_{2})\}\nonumber
\end{align}
\end{small}
since $\mathbb{E}(\|\widetilde{X}_{*}-\widetilde{X}_{*}^{'}\|^{2}_{2})\geq0$ always holds, we have
\begin{equation}
  \frac{1}{2}D_{E,\rho}(X_{I},X_{J})\leq\mathbb{E}(\|\widetilde{X}_{I}-\widetilde{X}_{J}\|^{2}_{2})\nonumber
\end{equation}
by substituting Eq.\ref{eq_energy_confusion} here, the proof is completed.
\end{proof}
\textbf{Remark}: From Lem.\ref{lem2}, one can observe that our EC can be viewed as an upper bound of GED, minimizing this upper bound function is equivalent to optimizing GED to some extend. Moreover, it seems that directly optimizing GED with $\rho=\|\cdot-\cdot\|^{2}_{2}$ is reasonable as well, since GED itself is a statistical distance between two probability distributions. However, by comparing EC with GED, we emphasize that directly minimizing GED will additionally make $\small{\mathbb{E}(\|\widetilde{X}_{I}-\widetilde{X}_{I}^{'}\|^{2}_{2})+\mathbb{E}(\|\widetilde{X}_{J}-\widetilde{X}_{J}^{'}\|^{2}_{2})}$ large, i.e. making points in the same class be far away from each other which violates the basic discrimination criterion of metric learning and will degrade the model into a noisy counterpart, it isn't what we desire. Therefore, GED cannot be directly applied here.

\textbf{Relation to MMD}: Let $k$ be a kernel on $\mathcal{Z}$, and let $P,Q\in{\mathcal{M}^{1}_{+}(\mathcal{Z})\bigcap{\mathcal{M}^{1/2}_{k}(\mathcal{Z})}}$. The \emph{maximum mean discrepancy}(MMD) $\gamma_{k}$ between $P$ and $Q$ is:
\begin{small}
\begin{align}\label{eq_mmd}
\gamma_{k}^{2}(P,Q)=&\|\mu_{k}(P)-\mu_{k}(Q)\|_{\mathcal{H}_{k}}^{2}=\|\mathbb{E}_{\widetilde{P}}k(\cdot,\widetilde{P})-\mathbb{E}_{\widetilde{Q}}k(\cdot,\widetilde{Q})\|^{2}_{\mathcal{H}_{k}}\nonumber\\
=&\mathbb{E}_{\widetilde{P}\widetilde{P}^{'}}k(\widetilde{P},\widetilde{P}^{'})+\mathbb{E}_{\widetilde{Q}\widetilde{Q}^{'}}k(\widetilde{Q},\widetilde{Q}^{'})-2\mathbb{E}_{\widetilde{P}\widetilde{Q}}k(\widetilde{P},\widetilde{Q})
\end{align}
\end{small}
\vspace{-1em}

where $\mu_{k}(*)$ is the kernel embedding, $\widetilde{P},\widetilde{P}^{'}\stackrel{i.i.d}{\thicksim}P$ and $\widetilde{Q},\widetilde{Q}^{'}\stackrel{i.i.d}{\thicksim}Q$. Then we have:
\begin{myLem}\label{lem3}
For two different class sets $X_{I},X_{J}\in\mathcal{M}^{1}_{+}(\mathcal{Z})\bigcap{\mathcal{M}^{1/2}_{k}(\mathcal{Z})}$, let $k$ be $degree$-$1$ homogeneous polynomial kernel, then:
\begin{equation}
  L_{ec}(\theta_{f};X_{I},X_{J})\geq\gamma_{k}^{2}(X_{I},X_{J})\nonumber
\end{equation}
\end{myLem}
\begin{proof}
Insert the \emph{distance-induced kernel} $k$ by corresponding $\rho$ from Lem.\ref{lem1} into Eq.\ref{eq_mmd}, and cancel out the terms dependant on a single random variable, we have:
\begin{footnotesize}
\begin{align}
\gamma_{k}^{2}(X_{I},X_{J})&=\frac{1}{2}\mathbb{E}_{\widetilde{X}_{I}\widetilde{X}_{I}^{'}}[\rho(\widetilde{X}_{I},z_{0})+\rho(\widetilde{X}_{I}^{'},z_{0})-\rho(\widetilde{X}_{I},\widetilde{X}_{I}^{'})]\nonumber\\
&+\frac{1}{2}\mathbb{E}_{\widetilde{X}_{J}\widetilde{X}_{J}^{'}}[\rho(\widetilde{X}_{J},z_{0})+\rho(\widetilde{X}_{J}^{'},z_{0})-\rho(\widetilde{X}_{J},\widetilde{X}_{J}^{'})]\nonumber\\
&-\mathbb{E}_{\widetilde{X}_{I}\widetilde{X}_{J}}[\rho(\widetilde{X}_{I},z_{0})+\rho(\widetilde{X}_{J},z_{0})-\rho(\widetilde{X}_{I},\widetilde{X}_{J})]\nonumber\\
=\mathbb{E}_{X_{I}X_{J}}&\rho(X_{I},X_{J})-\frac{1}{2}\mathbb{E}_{X_{I}X_{I}^{'}}\rho(X_{I},X_{I}^{'})-\frac{1}{2}\mathbb{E}_{X_{J}X_{J}^{'}}\rho(X_{J},X_{J}^{'})
\end{align}
\end{footnotesize}
\vspace{-1em}

i.e. $\gamma_{k}^{2}(X_{I},X_{J})=\frac{1}{2}D_{E,\rho}(X_{I},X_{J})$, since $k$ is $degree$-$1$ homogeneous polynomial kernel, from Prop.\ref{prop3} we have the corresponding generated $\rho=\|\cdot-\cdot\|^{2}_{2}$, then by using Lem.\ref{lem2}, we have $L_{ec}(\theta_{f};X_{I},X_{J})\geq\gamma_{k}^{2}(X_{I},X_{J})$.
\end{proof}
\textbf{Remark}: From Lem.\ref{lem3}, one can observe that our EC can also be viewed as an upper bound of MMD. Moreover, it seems that directly optimizing MMD with $degree$-$1$ homogeneous polynomial kernel, i.e. $\gamma_{k}^{2}=\|\mathbb{E}(\widetilde{X}_{I})-\mathbb{E}(\widetilde{X}_{J})\|^{2}_{\mathcal{H}_{k}}$, is reasonable as well, since many existing works employ this to pull two probability distributions closer, such as in transfer learning\cite{long2015learning,long2016deep,tzeng2014deep}. However, by expanding this $\gamma_{k}^{2}$, we have $\small{\gamma_{k}^{2}=\mathbb{E}({\widetilde{X}_{I}}^{T}\widetilde{X}^{'}_{I})+\mathbb{E}({\widetilde{X}_{J}}^{T}\widetilde{X}^{'}_{J})-2\mathbb{E}({\widetilde{X}_{I}}^{T}\widetilde{X}_{J})}$, and in this case, if we minimize $\gamma_{k}^{2}$ so as to pull different classes distributions closer and thus confuse the metric learning, we will additionally force $\mathbb{E}({\widetilde{X}_{I}}^{T}\widetilde{X}^{'}_{I})$+$\mathbb{E}({\widetilde{X}_{J}}^{T}\widetilde{X}^{'}_{J})$ to be small, which implicitly pushes the points within the same class further apart as their inner-products are getting small. This results also aren't what we desire and will degrade the model into a noisy counterpart. Therefore, MMD cannot be directly applied here as well.

\textbf{Remark Summary}: We theoretically derive the relations between our EC and both GED and MMD, and also reason about why they cannot be directly applied here even if they have been widely adopted in many machine learning tasks for measuring probability distributions. Thus, we will focus on 'confusing' the metric learning via our EC term.
\subsection{4.2 Energy Confused Adversarial Metric Learning}\label{sec_rEC}
The framework of ECAML can be generally applied to various metric learning objective functions, where we simultaneously train our Energy Confusion term and the distance metric term as follows:
\begin{equation}\label{eq_ecaml}
  \small{\min_{\theta_{f}}L=L_{m}(\theta_{f};T,D)+\lambda\sum_{I,J,I\neq{J}}{L_{ec}(\theta_{f};X_{I},X_{J})}}
  \vspace{-0.3em}
\end{equation}
where $\lambda$ is the trade-off hyper-parameter and class sets $X_{I},X_{J}$ are randomly chosen in the current minibatch. In order to demonstrate the effectiveness of the proposed ECAML framework, we develop various SOTA metric learning objective functions here, i.e. $\small{L_{m}(\theta_{f};T,D)}$:

\textbf{ECAML(Tri)}: For triplet-tuple $T$ and Euclidean measurement $D$, we employ\cite{Schroff2015FaceNet,wang2015unsupervised}:
\vspace{-0.5em}
\begin{equation}\label{eq_tri}
  \small{L_{m}(\theta_{f};T,D)=\sum_{i}^{N}[\|x_{i}-x_{i^{+}}\|^{2}_{2}-\|x_{i}-x_{i^{-}}\|^{2}_{2}+m]_{+}}
  \vspace{-0.5em}
\end{equation}
where the objective limits the distances of negative pairs larger than that of the positive pairs by margin $m$ and features $x_{i}$ is assumed to be on unit sphere, we experimentally find $m=0.1$ performs best.

\textbf{ECAML(N-Pair)}: For N-tuple $T$ and inner-product measurement $D$, we employ \cite{Sohn2016npair}:
\vspace{-0.5em}
\begin{equation}\label{eq_npair}
   \small{L_{m}(\theta_{f};T,D)=\sum_{i=1}^{N}\log(1+\sum_{j=1,y_{j}\neq{y_{i}}}^{N}exp(x_{i}^{T}x_{j}-x_{i}^{T}x_{i^{+}}))}
\end{equation}
where the objective limits the inner-product of each negative pair $x_{i}^{T}x_{j}$ smaller than that of the positive pair $x_{i}^{T}x_{i^{+}}$.

\textbf{ECAML(Binomial)}: For contrastive-tuple $T$ and cosine measurement $D$, we employ\cite{yi2014deep,Opitz_2017_ICCV}:
\vspace{-0.5em}
\begin{equation}\label{eq_bier}
  \small{L_{m}(\theta_{f};T,D)=\sum_{i,j}\log(1+e^{-(2s_{ij}-1)\alpha(D_{ij}-\beta)\eta_{ij}})}
  \vspace{-0.5em}
\end{equation}
where $s_{ij}=1$ if $x_{i},x_{j}$ are from the same class, otherwise $s_{ij}=0$, $\small{\alpha=2,\beta=0.5}$ are the scaling and translation parameters \emph{resp}, $\eta_{ij}$ is the penalty coefficient and is set to $1$ if $s_{ij}=1$, otherwise $\eta_{ij}=25$, $D_{ij}=\frac{x_{i}^{T}x_{j}}{\|x_{i}\|\|x_{j}\|}$.

Moreover, for numerical stability, we extend our EC to a logarithmic counterpart and thus Eq.\ref{eq_ecaml} becomes:
\vspace{-0.5em}
\begin{equation}\label{eq_log_ecaml}
  \small{\min_{\theta_{f}}L=L_{m}(\theta_{f};T,D)+\lambda\sum_{I,J,I\neq{J}}\log(1+{L_{ec}(\theta_{f};X_{I},X_{J})})}
  \vspace{-0.3em}
\end{equation}
\textbf{Discussion}: From Eq.\ref{eq_log_ecaml}, our ECAML is achieved by jointly training the conventional metric objective and the proposed Energy Confusion goal. These two terms form an adversarial learning scheme by optimizing the opposite objective functions. Specifically, $L_{m}$ acts as a 'defender' and $L_{ec}$ acts as an 'attacker', the attacker intends to confuse the metric so as to make it confound with the training data, while in order to correctly distinguish the training data, the defender has to learn more 'general' and complementary concepts. As the defending-attacking going, the learned embedding will be less likely to the prejudiced concepts and, thus successfully prevent the biased learning behavior and improve the generalization ability. Moreover, we experimentally find that the overfitting mainly appears at the fc layer, thus our EC term is only used to constrain the learning of fc layer.
\vspace{-0.5em}
\section{5. Experiments and Results}
\textbf{Implementation details}: Following many other works, e.g. \cite{oh2016deep,Sohn2016npair}, we choose the pretrained \emph{GooglenetV1}\cite{Szegedy2014Going} as our bedrock CNN and randomly initialized an added fully connected layer. If not specified, we set the embedding size as 512 throughout our experiments. We also adopt exactly the same data preprocessing method\cite{oh2016deep} so as to make fair comparisons with other works\footnote{Only the images in CARS dataset are preprocessed differently, see the detail underneath Tab.\ref{tab_car}}. For training, the optimizer is Adam\cite{kingma2014adam} with learning rate $\footnotesize{1e-5}$ and weight decay $\footnotesize{2e-4}$. The training iterations are $5k$(CUB), $10k$(CARS), $20k$(Stanford Online Products and In-Shop), \emph{resp}.
The new fc-layer is optimized with 10 times learning rate for fast convergence. Moreover, for fair comparison, we use minibatch of size 128 throughout our experiments, which is composed of $64$ random selected classes with two instances each class. Our work is implemented by caffe\cite{jia2014caffe}.

\textbf{Evaluation and datasets}: The same as many other works, the retrieval performance is evaluated by Recall@K metric. And following \cite{oh2016deep}, we evaluate the clustering performances via \emph{normalized mutual information}(NMI) and F$_{1}$ metrics. The input of NMI is a set of clusters $\Omega=\{\omega_{1},\ldots,\omega_{K}\}$ and the ground truth classes $\mathbb{C}=\{c_{1},\ldots,c_{K}\}$, where $\omega_{i}$ represents the samples that belong to the $i$th cluster, and $c_{j}$ is the set of samples with label $j$. NMI is defined as the ratio of mutual information and the mean entropy of clusters and the ground truth, NMI($\Omega,\mathbb{C}$)$=\frac{2I(\Omega,\mathbb{C})}{H(\Omega)+H(\mathbb{C})}$, and F$_{1}$ metric is the harmonic mean of precision and recall as follows F$_{1}=\frac{2PR}{P+R}$. Then our ECAML is evaluated over the widely used benchmarks with the standard \emph{zero-shot} evaluation protocol\cite{oh2016deep}:
\begin{small}
\begin{enumerate}[1)]
  \item \textbf{CARS}\cite{Krause20133D} contains 16,185 car images from 196 classes. We split the first 98 classes for training (8,054 images) and the rest 98 classes for testing (8,131 images).\vspace{-0.5em}
  \item \textbf{CUB}\cite{Wah2011The} includes 11,788 bird images from 200 classes.We use the first 100 classes for training (5,864 images) and the rest 100 classes for testing (5,924 images).\vspace{-0.5em}
  \item \textbf{Stanford Online Products}\cite{oh2016deep} has 11,318 classes for training (59,551 images) and the other 11,316 classes for testing (60,502 images).\vspace{-0.5em}
  \item \textbf{In-Shop}\cite{liu2016deepfashion} contains 3,997 classes for training(25,882 images) and the resting 3,985 classes for testing(28,760 images). The test set is partitioned into the query set of 3,985 classes(14,218 images) and the retrieval database set of 3,985 classes(12,612 images).
\end{enumerate}
\end{small}
\vspace{-0.5em}
\subsection{5.1 Ablation Experiments}
We show the primary results below and the qualitative analysis(embedding visualization) is placed in Supplementary.

\textbf{Regularization ability}:
\begin{figure}[t]
  \centering
  \vspace{-1em}
  \begin{minipage}{1\linewidth}
  \centering
  \includegraphics[width=1\linewidth]{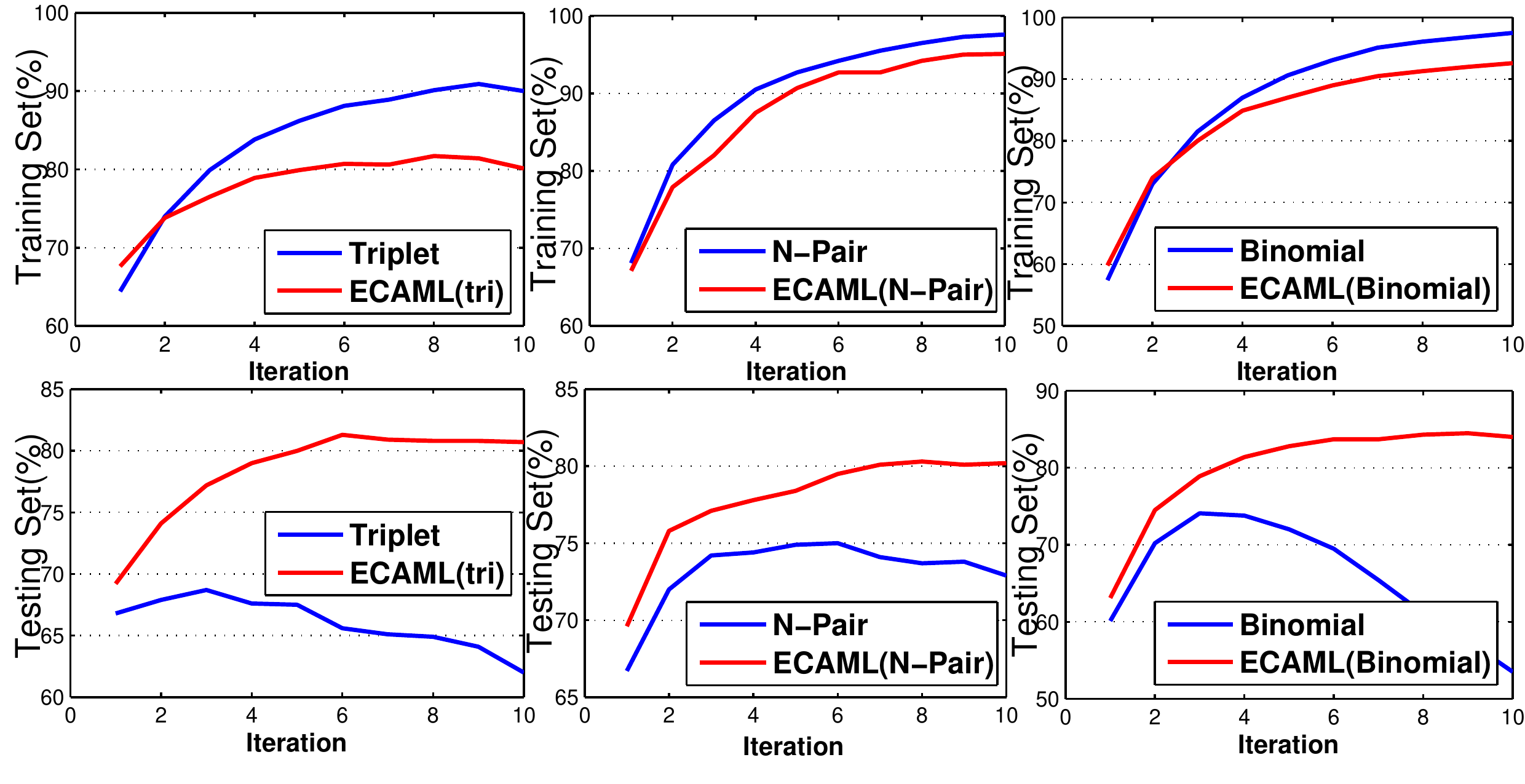}\\
  \vspace{-0.8em}\captionsetup{font={small}}
  \caption{Recall@1 curves on training(seen classes, top fig) and testing(unseen classes, bottom fig) sets over CARS dataset.}
  \label{fig_generalization}
  \end{minipage}
  \vspace{-1em}
\end{figure}
To demonstrate the regularization ability of our ECAML, we plot the R@1 retrieval result curves on training(\emph{seen}) and testing(\emph{unseen}) sets \emph{resp}, as in Fig.\ref{fig_generalization}. Specifically, for example, from the figures in left column, one can observe that the training curve of the conventional Triplet method rises quickly to a relatively high level but its testing curve only rises a little at first and then starts dropping to quite a low level, showing that the metric learned by conventional Triplet are more likely to over-fit the \emph{seen} classes and generalize worse to the \emph{unseen} classes in \emph{zero-shot} settings. Conversely, by employing our ECAML(Tri), the training result curve rises much slower than the original Triplet and stops rising at a relatively lower level ($80\%$ \emph{vs.} $90\%$), however, the testing cure of our ECAML(Tri) rises fast to quite a high level, more than $80\%$, implying that our ECAML(Tri) indeed serves as a regularization method and improves the generalization ability of the learned metric by suppressing the learning of biased metric over \emph{seen} classes caused by 'biased learning behavior'. Moreover, the similar phenomenon can be observed by ECAML(N-Pair,Binomial).
\begin{table}[t]
  \centering
  \resizebox{1\linewidth}{!}{
    \begin{tabular}{ccccccc}
    \hline
    \multicolumn{7}{c}{CARS R@1} \\
    \hline
      $\lambda$    & 0 (Triplet) & 0.001 & 0.01  & 0.02  & 0.1   & 1 \\
    ECAML(tri) & 68.3  & 74.6  & 80.1  & \textbf{81.0}    & 72.3  & 59.3 \\
    \hline
      $\lambda$    & 0 (N-Pair) & 0.1   & 0.2   & 0.3   & 0.4   & 0.5 \\
    ECAML(N-Pair) & 74.3  & 77.4  & 79.6  & \textbf{80.4}  & 78.6  & 73.7 \\
    \hline
      $\lambda$    & 0 (Binomial) & 0.01  & 0.1   & 0.13  & 0.15  & 0.5 \\
    ECAML(Binomial) & 74.2  &  76.3   & 83.1  & \textbf{84.5}  &  84.3   & 69.7 \\
    \hline
    \end{tabular}%
    }\captionsetup{font={small}}\vspace{-0.5em}
      \caption{Ablation experimental results on parameter $\lambda$.}
      \vspace{-1em}
  \label{tab_lamda}%
\end{table}%

\vspace{-0.5em}
\textbf{Ablation experiments on $\bf{\lambda}$}: To show the effectiveness of the parameter $\lambda$, here for simplicity, we just show the results of ECAML(tri,N-Pair,Binomial) with different $\lambda$ on CARS benchmark as in Tab.\ref{tab_lamda}. It can be observed that when $\lambda=0$ our ECAML degenerates into the corresponding conventional metric learning method and the performance is unsatisfactory, and as $\lambda$ increasing, the performances of ECAML(tri,N-Pair,Binomial) peak around \{$0.02, 0.3, 0.13$\} \emph{resp} and outperform the baselines (Triplet, N-Pair, Binomial) by a large margin, validating the effectiveness and importance of our ECAML.

\begin{table}[t]
  \centering
  \resizebox{0.9\linewidth}{!}{
    \begin{tabular}{lcccccc}
    \hline
    \multicolumn{7}{c}{CARS} \\
    \hline
       Methods   &  R@1  &  R@2     &   R@4    &   R@8    &   NMI    & F1 \\
    \hline
    Binomial-128 & 70.8  & 80.8  & 87.3  & 92.1  & 61.2  & 29.3 \\
    ECAML(Binomial)-128 & \textbf{79.6}    & \textbf{87.4}  & \textbf{91.8}  & \textbf{94.5}  & \textbf{64.6}  & \textbf{32.9} \\
    \hline
    Binomial-256 & 73.3  & 82.4  & 88.5  & 92.5  & 61.5  & 30.0 \\
    ECAML(Binomial)-256 & \textbf{82.0}    & \textbf{88.5}  & \textbf{92.5}  & \textbf{95.3}  & \textbf{66.4}  & \textbf{35.3} \\
    \hline
    Binomial-384 & 73.9  & 82.5  & 88.8  & 93.2  & 61.9  & 30.5 \\
    ECAML(Binomial)-384 & \textbf{83.5}  & \textbf{89.8}  & \textbf{93.5}  & \textbf{95.9}  & \textbf{67.0}    & \textbf{35.8} \\
    \hline
    Binomial-512 & 74.2  & 83.1  & 86.7  & 92.9  & 61.5  & 28.8 \\
    ECAML(Binomial)-512 & \textbf{84.5}  & \textbf{90.4}  & \textbf{93.8}  & \textbf{96.6}  & \textbf{68.4}  & \textbf{38.4} \\
    \hline
    \end{tabular}%
    }\captionsetup{font={small}}\vspace{-0.5em}
      \caption{Ablation experimental results on embedding size.}
      \vspace{-1.5em}
  \label{tab_dim}%
\end{table}%
\textbf{Ablation experiments on embedding size}: We also conduct quantitative experiments on embedding size with ECAML(Binomial). From Tab.\ref{tab_dim}, it can be observed that for the conventional Binomial metric learning method, most of the evaluation indexes' results (e.g. R@4, R@8, NMI and F$_{1}$) don't increase with the embedding size (from 128-dim to 512-dim) and even have a decrease trend, showing that the risk of overfitting increases with feature size and without robustness learning the performances of the learned embedding cannot be guaranteed even if its theoretical representation ability will increase with the feature size.
However, by employing our ECAML, the performances can be consistently improved and indeed increase with embedding size, demonstrating the importance and superiority of robust metric learning in ZSRC tasks.

\textbf{Ablation Study on Regularization Method} There are some other research works aiming at imposing regularization in the top layer of the whole network, such as label-smothing\cite{szegedy2016rethinking}, label-disturbing\cite{xie2016disturblabel} and Noisy-Softmax\cite{chen2017noisy}. However these methods are all designed for Softmax classifier layer and cannot be applied in the metric learning methods. Then, in order to show the effectiveness of our ECAML in the metric learning framework, we compare it with the commonly used 'Dropout' method. The dropout layer is placed after the CNN model. From Tab.\ref{tab_drop}, one can observe that although the dropout with ratio $0.25$ improves most of the performances over the baseline, the improvements are limited and not worthy of attention. However, in contrast to Dropout, our ECAML significantly surpasses the baseline model by a large margin. We conjecture that is because Dropout is not specially designed for the metric learning and the tested datasets are all fine-grained datasets in which simply depressing the neurons to be zero will largely affects the estimated distributions of these fine-grained classes regardless of the ratio value due to the small inter-class variations (for example, by using a smaller ratio (e.g. $0.1$) the performance will still be reduced). In summary, our ECAML regularization method is specially designed for the deep metric learning and indeed performs well.
\begin{table}[t]
  \centering
  \resizebox{0.9\linewidth}{!}{
    \begin{tabular}{ccccc}
    \hline
          & \multicolumn{4}{c}{CARS}   \\
          \hline
          & R@1   & R@2   & R@3   & R@4   \\
          \hline
    Binomial & 74.2  & 83.1  & 86.7  & 92.9  \\
    Dropout(Binomial,0.1) & 73.1  & 82.1  & \textcolor[rgb]{0,1,0}{88.6}  & 92.6 \\
    Dropout(Binomial,0.25) & \textcolor[rgb]{0, 1, 0}{74.5}  & \textcolor[rgb]{0, 1, 0}{83.3}  & 85.9  & 92.6  \\
    Dropout(Binomial,0.4) & 72.4  & 81.4  & \textcolor[rgb]{0, 1, 0}{87.5}  & 92.5 \\
    ECAML(Binomial) & \textcolor[rgb]{1, 0, 0}{84.5}  & \textcolor[rgb]{1, 0, 0}{90.4}  & \textcolor[rgb]{1, 0, 0}{93.8}  & \textcolor[rgb]{1, 0, 0}{96.6}   \\
    \hline
    \end{tabular}%
    }
    \vspace{-0.5em}\captionsetup{font={small}}
      \caption{Comparison with Dropout on CARS datasets. We have experimented Dropout with \{0.1,0.25,0.4\} ratio. The colored number represents improvement over the baseline Binomial method, specifically, red number indicates the improvement by our ECAML and green number by Dropout.}
      \vspace{-1em}
  \label{tab_drop}%
\end{table}%
\renewcommand\arraystretch{1.2}
\begin{table}[!t]
  \centering
  \resizebox{0.95\linewidth}{!}{
    \begin{tabular}{lccccccc}
    \hline
    \multicolumn{7}{c}{CARS} \\
    \hline
    Method  & R@1 &R@2 & R@4 & R@8 & NMI   & F1   \\
    \hline
    Lifted\small{\cite{oh2016deep}} & 49.0    & 60.3  & 72.1  & 81.5& 55.1  & 21.5 \\
    Clustering\small{\cite{songCVPR17}} & 58.1  & 70.6  & 80.3  & 87.8 & 59.0  & - \\
    Angular\small{\cite{wang2017deep}}  & 71.3  & 80.7  & 87.0    & 91.8& 62.4  & 31.8 \\
    ALMN\small{\cite{chen2018almn}}    & 71.6  & 81.3  & 88.2  & 93.4& 62.0    & 29.4 \\
    DAML\small{\cite{duan2018deep}} & 75.1 & 83.8 & 89.7 & 93.5 &\textcolor[rgb]{0, 0, 1}{66.0}&\textcolor[rgb]{0, 0, 1}{36.4}\\
    \hline\hline
    Triplet   & 68.3  & 78.3  & 86.2  & 91.7 &  59.2  &   26.2   \\
    ECAML(Tri)  & \textcolor[rgb]{0, 0, 1}{\emph{\textbf{81.0}}}    & \textcolor[rgb]{0, 0, 1}{\emph{\textbf{88.2}}}  & \textcolor[rgb]{0, 0, 1}{\emph{\textbf{92.8}}}  & \textcolor[rgb]{0, 0, 1}{\emph{\textbf{96.0}}} &  \emph{\textbf{65.7}}  & \emph{\textbf{33.0}}  \\
    \hline
    N-Pair   & 74.3  & 83.6  & 90.2  & 93.1 &   61.8   &  29.9 \\
    ECAML(N-Pair)  & \emph{\textbf{80.4}}  & \textcolor[rgb]{0, 0, 1}{\emph{\textbf{88.2}}}  & \emph{\textbf{92.4}}  & \emph{\textbf{95.8}} &  \textbf{\emph{64.6}}     & \textbf{\emph{32.7}}  \\
    \hline
    Binomial    & 74.2  & 83.1  & 86.7  & 92.9 &   61.5   &  28.8  \\
    ECAML(Binomial)    & \textcolor[rgb]{1, 0, 0}{\emph{\textbf{84.5}}}    & \textcolor[rgb]{1, 0, 0}{\emph{\textbf{90.4}}}    & \textcolor[rgb]{1, 0, 0}{\emph{\textbf{93.8}}}  & \textcolor[rgb]{1, 0, 0}{\emph{\textbf{96.6}}} &    \textcolor[rgb]{1, 0, 0}{\emph{\textbf{68.4}}}   & \textcolor[rgb]{1, 0, 0}{\emph{\textbf{38.4}}}  \\
    \hline
    \end{tabular}%
    }\vspace{-0.5em}\captionsetup{font={scriptsize}}
      \caption{Comparisons(\%) with state-of-the-arts on CARS\cite{Krause20133D}. $\lambda$ for ECAML(tri, N-Pair, Binomial) are \{$0.02,0.3,0.13$\} \emph{resp}. Here, the images are directly resized to 256x256, which are different from\cite{oh2016deep}, then a 227x227 random region is cropped.}\vspace{-0.8em}
  \label{tab_car}%
\end{table}%
\begin{table}[!t]
  \centering
   \resizebox{0.95\linewidth}{!}{
    \begin{tabular}{lccccccc}
    \hline
    \multicolumn{7}{c}{CUB} \\
    \hline
    Method & R@1 & R@2 & R@4 & R@8 & NMI   & F1 \\
    \hline
    Lifted\small{\cite{oh2016deep}} & 47.2  & 58.9  & 70.2  & 80.2 &  56.2   &  22.7 \\
    Clustering\small{\cite{songCVPR17}} & 48.2  & 61.4  & 71.8  & 81.9 &  59.2   & -  \\
    Angular\small{\cite{wang2017deep}}  & \textcolor[rgb]{0, 0, 1}{53.6}  &65.0    & 75.3  & 83.7 & 61.0    & \textcolor[rgb]{0, 0, 1}{30.2}\\
    ALMN\small{\cite{chen2018almn}}  & 52.4  & 64.8  & 75.4  & 84.3 & 60.7  & 28.5 \\
    DAML\small{\cite{duan2018deep}} & 52.7 & \textcolor[rgb]{0, 0, 1}{65.4} & 75.5 & 84.3 &\textcolor[rgb]{0, 0, 1}{61.3}&29.5\\
    \hline\hline
    Triplet   & 49.5  & 61.7  & 73.2  & 82.5 &  57.2  & 24.1\\
    ECAML(Tri)  & \textbf{\emph{53.4}}  & \textbf{\emph{64.7}}  & \textbf{\emph{75.1}}  & \emph{\textbf{84.7}} &   \emph{\textbf{60.1}}   &  \emph{\textbf{26.9}} \\
    \hline
    N-Pair  & 51.9  & 63.3  & 73.9  & 83.0 &  59.7  & 26.5 \\
    ECAML(N-Pair)  & \textbf{\emph{53.2}}  & \textbf{\emph{65.1}}  & \textcolor[rgb]{0, 0, 1}{\textbf{\emph{75.9}}}  & \textcolor[rgb]{0, 0, 1}{\textbf{\emph{84.9}}} & \emph{\textbf{60.4}}      & \emph{\textbf{28.5}} \\
    \hline
    Binomial     & 52.9  & 65.0    & 75.4  & 83.6 &    59.0   &  26.5 \\
    ECAML(Binomial)    & \textcolor[rgb]{1.000, 0.000, 0.000}{\emph{\textbf{55.7}}}  & \textcolor[rgb]{1.000, 0.000, 0.000}{\emph{\textbf{66.5}}}  & \textcolor[rgb]{1.000, 0.000, 0.000}{\emph{\textbf{76.7}}}  & \textcolor[rgb]{1.000, 0.000, 0.000}{\emph{\textbf{85.1}}} &   \textcolor[rgb]{1.000, 0.000, 0.000}{\emph{\textbf{61.8}}}   &  \textcolor[rgb]{1.000, 0.000, 0.000}{\emph{\textbf{30.5}}}  \\
    \hline
    \end{tabular}%
    }\vspace{-0.5em}\captionsetup{font={scriptsize}}
      \caption{Comparisons(\%) with state-of-the-arts on CUB\cite{Wah2011The}. $\lambda$ for ECAML(tri, N-Pair, Binomial) are \{$0.02,0.3,0.13$\} \emph{resp}.}\vspace{-0.8em}
  \label{tab_cub}%
\end{table}%
\begin{table}[!t]
  \centering
  \resizebox{0.95\linewidth}{!}{
    \begin{tabular}{lccccccc}
    \hline
    \multicolumn{7}{c}{Stanford Online Products} \\
    \hline
    Method  &R@1 &R@10 & R@100 & R@1000 & NMI   & F1  \\
    \hline
    Lifted\small{\cite{oh2016deep}}  & 62.1  & 79.8  & 91.3  & 97.4&  87.4  &  24.7 \\
    Clustering\small{\cite{songCVPR17}}   & 67.0    & 83.7  & 93.2  & -&  \textcolor[rgb]{0, 0, 1}{89.5}  &  - \\
    Angular\small{\cite{wang2017deep}}  & \textcolor[rgb]{0, 0, 1}{70.9}  & \textcolor[rgb]{0, 0, 1}{85.0}    & \textcolor[rgb]{0, 0, 1}{93.5}  & \textcolor[rgb]{1, 0, 0}{98.0} & 87.8  & 26.5 \\
    ALMN\small{\cite{chen2018almn}}  & 69.9  & 84.8  & 92.8  & - & -     & - \\
    DAML\small{\cite{duan2018deep}} & 68.4 & 83.5 & 92.3 & - &89.4&\textcolor[rgb]{0, 0, 1}{32.4}\\
    \hline\hline
    Triplet & 57.9  & 75.6  & 88.5  & 96.3 & 86.4  & 20.8\\
    ECAML(Tri)   & \emph{\textbf{64.9}}  & \emph{\textbf{80.0}}    & \emph{\textbf{90.5}}  & \emph{\textbf{96.9}} & \emph{\textbf{87.0}}    & \emph{\textbf{23.3}}\\
    \hline
    N-Pair     & 68.0  &   84.0    &   93.1   & 97.8 &  87.6  &  25.8 \\
    ECAML(N-Pair)    & \emph{\textbf{69.8}}  & \emph{\textbf{84.7}}      &  \emph{\textbf{93.2}}     & \emph{\textbf{97.8}} &   \emph{\textbf{88.0}}    &   \emph{\textbf{27.2}}  \\
    \hline
    Binomial      & 68.5  &  84.0  &   93.1    & 97.7 &   88.5  & 29.9  \\
    ECAML(Binomial)   & \textcolor[rgb]{1, 0, 0}{\emph{\textbf{71.3}}}  &   \textcolor[rgb]{1, 0, 0}{\emph{\textbf{85.6}}}    &    \textcolor[rgb]{1, 0, 0}{\emph{\textbf{93.6}}}   & \textcolor[rgb]{1, 0, 0}{\emph{\textbf{98.0}}} & \textcolor[rgb]{1, 0, 0}{\emph{\textbf{89.9}}}  & \textcolor[rgb]{1, 0, 0}{\emph{\textbf{32.8}}} \\
    \hline
    \end{tabular}%
    }\vspace{-0.5em}\captionsetup{font={scriptsize}}
      \caption{Comparisons(\%) with state-of-the-arts on Stanford Online Products\cite{oh2016deep}. $\lambda$ for ECAML(tri, N-Pair, Binomial) are \{$0.002,0.03,0.013$\} \emph{resp}.}\vspace{-0.8em}
  \label{tab_sop}%
\end{table}%
\begin{table}[!t]
  \centering
  \resizebox{0.95\linewidth}{!}{
    \begin{tabular}{lcccccc}
    \hline
    \multicolumn{7}{c}{In-Shop} \\
    \hline
    Method & R@1   & R@10  & R@20  & R@30  & R@40  & R@50  \\
    \hline
    FashionNet\small{\cite{liu2016deepfashion}} & 53    & 73    & 76    & 77    & 79    & 80 \\
    HDC\small{\cite{Yuan_2017_ICCV}}   & 62.1  & 84.9  & 89.0    & 91.2  & 92.3  & 93.1 \\
    BIER\small{\cite{Opitz_2017_ICCV}}  & 76.9  & 92.8  & 95.2  & 96.2  & 96.7  & 97.1 \\
    \hline\hline
    Triplet & 64.4  & 87.1  & 91.0    & 92.7  & 93.9  & 94.8  \\
    ECAML(Tri) & \emph{\textbf{68.0}}    & \emph{\textbf{89.9}}  & \emph{\textbf{93.3}}  & \emph{\textbf{94.8}}  & \emph{\textbf{95.7}}  & \emph{\textbf{96.3}}  \\
    \hline
    N-Pair & 78.2  & 94.3  & 96.0    & 96.9  & 97.4  & 97.7  \\
    ECAML(N-Pair) &  \emph{\textbf{79.8}}  &    \textcolor[rgb]{0, 0, 1}{\emph{\textbf{94.6}}}   &   \emph{\textbf{96.1}}    &    \emph{\textbf{97.0}}   &     97.4  &   97.7 \\
    \hline
    Binomial & \textcolor[rgb]{0, 0, 1}{81.7}  & 94.5  & \textcolor[rgb]{0, 0, 1}{96.2}  & \textcolor[rgb]{0, 0, 1}{97.2}  & \textcolor[rgb]{0, 0, 1}{97.6}  & \textcolor[rgb]{0, 0, 1}{97.9} \\
    ECAML(Binomial) & \textcolor[rgb]{1, 0, 0}{\emph{\textbf{83.8}}}  & \textcolor[rgb]{1, 0, 0}{\emph{\textbf{95.1}}}  & \textcolor[rgb]{1, 0, 0}{\emph{\textbf{96.6}}}  & \textcolor[rgb]{1, 0, 0}{\emph{\textbf{97.3}}}  & \textcolor[rgb]{1, 0, 0}{\emph{\textbf{97.7}}}  & \textcolor[rgb]{1, 0, 0}{\emph{\textbf{98.0}}}  \\
    \hline
    \end{tabular}%
    }\vspace{-0.5em}\captionsetup{font={scriptsize}}
      \caption{Comparisons(\%) with state-of-the-arts on In-shop\cite{liu2016deepfashion}. $\lambda$ for ECAML(tri, N-Pair, Binomial) are \{$0.002,0.03,0.013$\} \emph{resp}.}\vspace{-2em}
  \label{tab_shop}%
\end{table}%
\vspace{-0.5em}
\subsection{5.2 Comparison with State-of-the-art}
To highlight the significance of our ECAML framework, we compare with the aforementioned corresponding baseline methods, i.e. the wildly used Triplet\cite{Schroff2015FaceNet}, N-Pair\cite{Sohn2016npair} and Binomial\cite{yi2014deep}, moreover, we also compare ECAML with other SOTA methods. The experimental results over CUB, CARS, Stanford Online Products and In-shop are in Tab.\ref{tab_car}-Tab.\ref{tab_shop} \emph{resp}, bold number indicates improvement over baseline method, red and blue number show the best and second best results \emph{resp}. From these tables, one can observe that our ECAML consistently improves the performances of original metric learning methods (i.e. Triplet, N-Pair and Binomial) on all the benchmark datasets by a large margin, demonstrating the necessity of explicitly enhancing the generalization ability of the learned metric and validating the universality and effectiveness of our ECAML. Furthermore, our ECAML(Binomial) also surpasses all the listed state-of-the-art approches. In summary, learning 'general' concepts by avoiding the biased learning behavior is more important in ZSRC tasks and the generalization ability of the optimized metric heavily affects the performance of conventional metric learning methods.
\vspace{-0.5em}
\section{6. Conclusion}
In this paper, we propose the \emph{Energy Confused Adversarial Metric Learning} (ECAML) framework, a generally applicable methods to various conventional metric learning approaches, for ZSRC tasks by explicitly intensifying the generalization ability within the learned embedding with the help of our Energy Confusion term. Extensive experiments on the popular ZSRC benchmarks(CUB, CARS, Stanford Online Products and In-Shop) demonstrate the significance and necessity of our idea of learning metric with good generalization by energy confusion.

\textbf{Acknowledgments}: This work was partially supported by the National Natural Science Foundation of China under Grant Nos. 61573068 and 61871052, Beijing Nova Program under Grant No. Z161100004916088, and supported by BUPT Excellent Ph.D. Students Foundation CX2019307.
{\tiny
\bibliographystyle{aaai}
\bibliography{Bibliography-File}
}

\end{document}